\documentclass[letterpaper,10pt,conference]{ieeeconf}

\IEEEoverridecommandlockouts		
\makeatletter
\let\NAT@parse\undefined
\makeatother

\usepackage{cite}
\usepackage{amsmath,amssymb,amsfonts}
\usepackage{algorithmic}
\usepackage{graphicx}
\usepackage{textcomp}
\usepackage{xcolor}
\def\BibTeX{{\rm B\kern-.05em{\sc i\kern-.025em b}\kern-.08em
		T\kern-.1667em\lower.7ex\hbox{E}\kern-.125emX}}

\graphicspath{ {./figures/} }
\usepackage{float}			
\usepackage{tikz}			
\usetikzlibrary{cd}			
\usepackage[caption=false,font=footnotesize]{subfig}	

\usepackage{mathtools}				
\DeclareMathOperator{\vect}{vec}	
\DeclarePairedDelimiter\ceiling{\lceil}{\rceil}	

\usepackage{amsthm}
\theoremstyle{plain}
\newtheorem{theorem}{Theorem}
\newtheorem{proposition}{Proposition}
\newtheorem{corollary}{Corollary}

\theoremstyle{definition}
\newtheorem{definition}{Definition}
\newtheorem{example}{Example}
\newtheorem{assumption}{Assumption}
\theoremstyle{remark}
\newtheorem{remark}{Remark}

\usepackage{bookmark}
\bookmarksetup{depth=subsection}

\renewcommand{\thesubsubsection}{\arabic{subsubsection}}

\usepackage{hyperref}		
\hypersetup{%
	colorlinks=true,%
	linkcolor=black,%
	urlcolor=black,%
	citecolor=black,%
	filecolor=black%
}							
\usepackage[all]{hypcap}	

\begin{document}

\title{%
	Global Optimality Guarantees for Nonconvex Unsupervised Video Segmentation%
	\thanks{%
		Brendon G.\ Anderson is with the Department of Mechanical Engineering at University of California, Berkeley. Email: \href{mailto:bganderson@berkeley.edu}{\texttt{bganderson@berkeley.edu}}.}%
	\thanks{%
		Somayeh Sojoudi is with the Department of Electrical Engineering and Computer Sciences, the Department of Mechanical Engineering, and the Tsinghua-Berkeley Shenzhen Institute at University of California, Berkeley. Email: \href{mailto:sojoudi@berkeley.edu}{\texttt{sojoudi@berkeley.edu}}.}%
	\thanks{%
		This work was supported by ONR Award N00014-18-1-2526, NSF Award 1808859, and AFOSR Award FA9550-19-1-0055.}%
}

\author{%
	Brendon G.\ Anderson and Somayeh Sojoudi%
}

\maketitle

\begin{abstract}
	In this paper, we consider the problem of unsupervised video object segmentation via background subtraction. Specifically, we pose the nonsemantic extraction of a video's moving objects as a nonconvex optimization problem via a sum of sparse and low-rank matrices. The resulting formulation, a nonnegative variant of robust principal component analysis, is more computationally tractable than its commonly employed convex relaxation, although not generally solvable to global optimality. In spite of this limitation, we derive intuitive and interpretable conditions on the video data under which the uniqueness and global optimality of the object segmentation are guaranteed using local search methods. We illustrate these novel optimality criteria through example segmentations using real video data.
\end{abstract}

\section{Introduction}
\label{sec: introduction}
One of the most fundamental problems in computer vision and machine learning is that of video object segmentation. In this domain, the general goal is to distinguish and extract objects of interest from the rest of the video's content. Visual segmentation algorithms take on a variety of different tasks and forms. For instance, semantic segmentation tackles the problem of assigning each extracted object to a certain cluster or predefined class, and supervised (or semi-supervised) methods are endowed with one or more ground truth extractions or annotations \cite{long_2015}. This wide range of methodologies makes video segmentation suitable for many applications, such as surveillance systems, traffic monitoring, and gesture recognition, and therefore video object segmentation remains an active and challenging area of research \cite{bouwmans_2019,perazzi_2016}.

This paper is concerned with nonsemantic and unsupervised video segmentation via background subtraction; the task of extracting moving objects from a video's background using static cameras. Traditional techniques for moving object segmentation typically use Gaussian mixture models (GMM), which offer simple models but lack robustness \cite{stauffer_1999,goyal_2018}. Neural networks have also found popularity due to the balance they strike between performance and computational efficiency \cite{culibrk_2006}. However, due to their nonconvex nature, neural networks do not generally possess guarantees on the global optimality of their resulting segmentation \cite{bouwmans_2014}.

Recently, much attention has been placed on approaches based on robust principal component analysis (RPCA), which model the video as the sum of low-rank and sparse matrices. Perhaps the most notable of these methods is Principal Component Pursuit (PCP) introduced in the seminal paper by Cand\`es et al.\ \cite{candes_2011}. Although the convexified approach in PCP provides conditions under which exact recovery of the sparse components is guaranteed, its use of lifted variables results in scalability and computational hindrances \cite{bouwmans_2014}.

In order to tackle large-scale segmentation problems, lower-dimensional nonconvex formulations such as nonnegative robust principal component analysis (NRPCA) and robust nonnegative matrix factorization (RNMF) have been proposed \cite{fattahi_2018,zhang_2011,guan_2012,bouwmans_2017}. These nonconvex approaches often permit parallelization, lending themselves to lowered computational cost and scalability to larger problems \cite{chi_2018}. Furthermore, the nonnegative nature of grayscale pixel values is explicitly embedded in modern methods like NRPCA and RNMF, unlike many of the more traditional techniques. Although these nonconvex formulations have been empirically shown to have performance on par with the popular PCP method, previous works have focused on local optimality of the resulting video segmentations, often solved for by alternating over the subproblems that are convex in the variables separately \cite{zhang_2011,guan_2012}.

In this work, we aim to supplement the strong empirical and computational properties of video segmentation via nonconvex NRPCA by providing intuitive and interpretable global optimality guarantees. These guarantees target two key aspects of moving object segmentation. First, they promise global solutions when using local search algorithms, such as stochastic gradient descent or its variants. The computational efficiency of these simple algorithms is paramount in large-scale machine learning problems \cite{bottou_2010,bottou_2012,johnson_2013}, e.g., those with high-resolution video data. Second, safety-critical video segmentation applications, such as autonomous driving \cite{ramanagopal_2018} and medical imaging \cite{zou_2012}, demand global optimality guarantees to promise consistent performance and safety margins. With the recent influx of studies on spurious local minima of nonconvex optimization problems \cite{bhojanapalli_2016,ge_2016,josz_2018,zhang_2018}, we approach this problem by exploiting new results on the benign landscape of rank-1 NRPCA \cite{fattahi_2018}. Under this framework, we propose criteria under which the video segmentation is guaranteed to be unique and globally optimal.

The remainder of this paper is structured as follows. In Section \ref{sec: problem_statement}, we describe the problem and introduce our terminology and notations. In Section \ref{sec: object_embedding}, we show that the problem can be simplified to one in which the moving objects consist of elementary shapes. Then, in Sections \ref{sec: conditions_for_connectivity} and \ref{sec: conditions_for_identifiability}, we derive conditions on video data under which global optimality guarantees can be made. Finally, we perform numerical experiments and make concluding remarks in Sections \ref{sec: numerical_experiments} and \ref{sec: conclusions}.

\section{Problem Statement}
\label{sec: problem_statement}
Throughout the paper, the nonnegative and positive orthants of $\mathbb{R}^n$ are denoted by $\mathbb{R}^n_+ = \{x\in\mathbb{R}^n : x_i \ge 0, ~ i\in\{1,2,\dots,n\}\}$ and $\mathbb{R}^n_{++} = \{x\in\mathbb{R}^n : x_i > 0, ~ i\in\{1,2,\dots,n\}\}$, respectively. We use analogous definitions for $\mathbb{Z}_+^n$ and $\mathbb{Z}_{++}^n$. Furthermore, we use the symbol $1_n$ to represent the $n$-vector of all ones.

Consider a video sequence of $d_f$ frames, each being $d_m$ pixels tall and $d_n$ pixels wide, where $d_f,d_m,d_n\in\mathbb{Z}_{++}$. We denote the video frames by the matrices $X^{(k)} \in \mathbb{R}^{d_m\times d_n}$, where $k\in K \coloneqq \{1,2,\dots,d_f\}$. By defining the \textit{pixel set} as $\Pi = \{1,2,\dots,d_m\}\times\{1,2,\dots,d_n\}$, the pixels of a grayscale video are given by
\begin{equation*}
X_{ij}^{(k)} \in \mathcal{X}\subseteq\mathbb{R}, \quad (i,j)\in\Pi,~k\in K,
\end{equation*}
where conventionally, $\mathcal{X}=\{0,1,\dots,255\}$ or $\mathcal{X}=[0,1]$. In this work, we scale the pixel values to an interval $\mathcal{X}=[X_{\textup{black}},X_{\textup{white}}]\subseteq\mathbb{R}_+$, for technical reasons explained later. Vectorizing each frame of the video, we form the data matrix
\begin{equation*}
X = \begin{bmatrix}
\vect X^{(1)} & \vect X^{(2)} & \cdots & \vect X^{(d_f)}
\end{bmatrix} \in\mathbb{R}^{m\times n},
\end{equation*}
where $m=d_md_n$ and $n=d_f$. Note that $\vect \colon \mathbb{R}^{d_m\times d_n} \to \mathbb{R}^{d_md_n}$ converts each frame into an equivalent extended vector, so that the single matrix $X$ captures all of the video's information. We also define the \textit{measurement set} as $\Omega = \{1,2,\dots,m\}\times\{1,2,\dots,n\}$.

We choose to model the video data matrix as the sum of two components. The first component is chosen to be a nonnegative rank-1 matrix, used to capture the relatively static behavior of the video's background. The second component is a sparse matrix, taken to represent the dynamic foreground (i.e., the moving objects). Under this model, we seek the decomposition
\begin{equation}
X \approx uv^{\top} + S, \label{eqn: decomposition}
\end{equation}
where $u\in\mathbb{R}_+^m$ and $v\in\mathbb{R}_+^n$, and $S\in\mathbb{R}^{m\times n}$ is sparse. This can be solved for through the following nonconvex, nonnegative $l_1$-minimization problem, termed in the literature as \textit{nonnegative robust principal component analysis} (NRPCA) \cite{fattahi_2018}:
\begin{equation}
\begin{array}{rl}
\textup{minimize} & \|X-uv^{\top}\|_1 + \lambda\left| u^{\top}u - v^{\top}v \right| \\
\textup{subject to} & u\in\mathbb{R}_+^m,~v\in\mathbb{R}_+^n.
\end{array} \label{eqn: nrpca_problem}
\end{equation}
This is a nonconvex problem that may generally have spurious local minima, i.e., those points a local search algorithm may find which do not correspond to the globally optimal solution. Note that we enforce nonnegativity of the optimization variables $u$ and $v$, yielding natural interpretations as the video's nominal background pattern and its associated scalings in each frame, respectively. Furthermore, we have added a regularization term with tuning parameter $\lambda\in\mathbb{R}_{++}$ to our formulation, since the unregularized objective is invariant to scaling. In other words, if $(u^*,v^*)$ minimizes the unregularized problem, then so will $(\alpha u^*,\frac{1}{\alpha}v^*)$ for every $\alpha \in \mathbb{R}_{++}$. Therefore, under regularization, the unique solution should be the pair $(u^*,v^*)$ for which $\|u^*\|_2=\|v^*\|_2$.

Under the decomposition (\ref{eqn: decomposition}), we define the video's \textit{background set} and \textit{foreground set} as $B = \{(h,k) \in\Omega : S_{hk}=0\}$ and $F = \Omega\setminus B$, respectively. Accordingly, two bipartite graphs can be introduced, the \textit{background graph} $\mathcal{G}_{m,n}(B)$ having edge set $B$, and the \textit{foreground graph} $\mathcal{G}_{m,n}(F)$ having edge set $F$. The first vertex set of each graph corresponds to pixel numbers: $V_u = \{1,2,\dots,m\}$. The second vertex set associates with frame numbers: $V_v = \{m+1,m+2,\dots,m+n\}$. A toy example of these graphs follows.

\begin{example}
	\label{ex: graph_examples}
	Suppose that a video has frames given by $X^{(1)} = \left[\begin{smallmatrix}
	256 & 1 \\
	256 & 1
	\end{smallmatrix}\right]$ and $X^{(2)} = \left[\begin{smallmatrix}
	1 & 256 \\
	256 & 256
	\end{smallmatrix}\right]$, where elements of 256 represent background. Then, the data matrix is
	\begin{equation*}
	X = \begin{bmatrix}
	256 & 1 \\
	256 & 256 \\
	1 & 256 \\
	1 & 256
	\end{bmatrix},
	\end{equation*}
	and the foreground and background sets are, respectively, $F = \{(3,1),(4,1),(1,2)\}$ and $B = \{1,2,3,4\}\times\{1,2\}\setminus F$. The corresponding graphs are shown in Fig.\ \ref{fig: graph_examples}.
	
	\begin{figure}[ht]
		\centering
		\subfloat[]{%
			\centering
			\begin{tikzcd}[%
				,cramped,row sep=tiny,column sep=huge,cells={nodes={circle,draw}}%
				,every arrow/.append style={dash,thick}%
				,ampersand replacement=\&]
				1 \arrow[rd] \& 5 \\
				2 \& 6 \\
				3 \arrow[ruu] \& \\
				4 \arrow[ruuu] \& \\
			\end{tikzcd}%
		}
		\hfil
		\subfloat[]{%
			\centering
			\begin{tikzcd}[%
				,cramped,row sep=tiny,column sep=huge,cells={nodes={circle,draw}}%
				,every arrow/.append style={dash,thick}%
				,ampersand replacement=\&]
				1 \arrow[r] \& 5 \\
				2 \arrow[ru] \arrow[r] \& 6 \\
				3 \arrow[ru] \& \\
				4 \arrow[ruu] \& \\
			\end{tikzcd}%
		}
		\caption{Example graphs $\mathcal{G}_{m,n}(F)$ (a), and $\mathcal{G}_{m,n}(B)$ (b).}
		\label{fig: graph_examples}
	\end{figure}
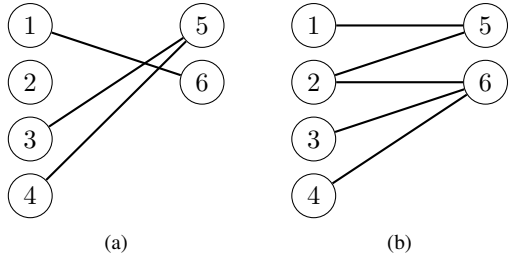
\end{example}

Now, a remarkable property of the nonconvex and nonsmooth problem (\ref{eqn: nrpca_problem}) is that, under certain conditions on the problem data, the optimization landscape is \textit{benign}, i.e., there are no spurious local minima, and the global minimum is unique \cite{fattahi_2018}. This permits the use of simple local search algorithms to solve (\ref{eqn: nrpca_problem}) to global optimality. The nonconservative sufficient conditions for benign landscape follow:
\begin{align}
\textup{Connectivity:} \quad & \mathcal{G}_{m,n}(B) \textup{ is connected}. \label{eqn: connectivity} \\
\textup{Identifiability:} \quad & \delta(\mathcal{G}_{m,n}(B)) > \frac{48}{c^2}\kappa(w^*)^4\Delta(\mathcal{G}_{m,n}(F)). \label{eqn: identifiability}
\end{align}
In these expressions, we denote the globally optimal solution of (\ref{eqn: nrpca_problem}) as $w^*=(u^*,v^*) \in\mathbb{R}^{m+n}$, the condition number (maximum element divided by minimum element) of a vector in the positive orthant as $\kappa(\cdot)$, maximum degree of a graph as $\Delta(\cdot)$, and minimum degree of a graph as $\delta(\cdot)$. The value $c$ is a constant that depends on problem data, which will be discussed in more detail later.

The problem to be addressed is as follows: \textit{When do videos satisfy the conditions (\ref{eqn: connectivity}) and (\ref{eqn: identifiability}) to guarantee a benign landscape for the optimization problem (\ref{eqn: nrpca_problem})?} In other words, the goal is to determine conditions on the size, shape, and speed of a moving object to provide theoretical guarantees for the unique and globally optimal foreground segmentation of a video. We begin by showing that the problem can be simplified to one with elementary foreground shapes through the notion of object embedding.

\section{Object Embedding}
\label{sec: object_embedding}
In this section, we consider two videos with identical backgrounds, each having one moving object (though the results are naturally generalized to multi-object videos). We are interested in the case that the moving object of one video can be completely covered by the moving object of the other video in each frame. Here is the question of interest: \textit{If the video with the larger moving object satisfies the conditions (\ref{eqn: connectivity}) and (\ref{eqn: identifiability}) for benign landscape of (\ref{eqn: nrpca_problem}), does the video with the smaller object also satisfy these conditions?} To answer this question precisely, let us start with the following definition.

\begin{definition}[Embedding]
	\label{def: embedding}
	Consider two videos $\mathcal{O}$ and $\mathcal{R}$ having the same background $uv^{\top}$, i.e., $X_{\mathcal{O}} = uv^{\top} + S_{\mathcal{O}}$ and $X_{\mathcal{R}} = uv^{\top} + S_{\mathcal{R}}$. We say that object $F_{\mathcal{O}}$ is \textit{embedded} in object $F_{\mathcal{R}}$ if the foreground of video $\mathcal{O}$ is a subset of that of video $\mathcal{R}$ in every frame; if
	\begin{equation*}
	F_{\mathcal{O}}^{(k)} \subseteq F_{\mathcal{R}}^{(k)} \textup{ for all } k\in K,
	\end{equation*}
	where $F_{\mathcal{O}}^{(k)} = \{(i,j)\in\Pi : (S_{\mathcal{O}})_{hk} \ne 0, ~ h = (j-1)d_m + i\}$, and similarly for $F_{\mathcal{R}}^{(k)}$.
\end{definition}

It is desirable to show that the answer to our earlier question is affirmative. We prove these implications in the following two propositions.

\begin{proposition}[Embedded connectivity]
	\label{prop: embedded_connectivity}
	If object $F_{\mathcal{O}}$ is embedded in object $F_{\mathcal{R}}$ and video $\mathcal{R}$ satisfies the connectivity condition (\ref{eqn: connectivity}), then video $\mathcal{O}$ also satisfies the connectivity condition.
\end{proposition}
\begin{proof}
	Since $F_{\mathcal{O}}$ is embedded in $F_{\mathcal{R}}$, we have for all $k\in K$ that $F_{\mathcal{O}}^{(k)} \subseteq F_{\mathcal{R}}^{(k)}$, which implies $\Pi \setminus F_{\mathcal{R}}^{(k)} \subseteq \Pi \setminus F_{\mathcal{O}}^{(k)}$. This shows that the background of video $\mathcal{R}$ is a subset of that of video $\mathcal{O}$, i.e., $B_{\mathcal{R}}^{(k)} \subseteq B_{\mathcal{O}}^{(k)}$ for all $k\in K$, which gives $B_{\mathcal{R}} \subseteq B_{\mathcal{O}}$. Therefore, we have that $\mathcal{G}_{m,n}(B_{\mathcal{R}})$ is a spanning subgraph of $\mathcal{G}_{m,n}(B_{\mathcal{O}})$. Since $\mathcal{G}_{m,n}(B_{\mathcal{R}})$ is connected by our assumption, so must be $\mathcal{G}_{m,n}(B_{\mathcal{O}})$, as desired.
\end{proof}

\begin{proposition}[Embedded identifiability]
	\label{prop: embedded_identifiability}
	If object $F_{\mathcal{O}}$ is embedded in object $F_{\mathcal{R}}$ and video $\mathcal{R}$ satisfies the identifiability condition (\ref{eqn: identifiability}), then video $\mathcal{O}$ also satisfies the identifiability condition.
\end{proposition}
\begin{proof}
	Since $F_{\mathcal{O}}$ is embedded in $F_{\mathcal{R}}$, we have for all $k\in K$ that $F_{\mathcal{O}}^{(k)} \subseteq F_{\mathcal{R}}^{(k)}$, which implies $F_{\mathcal{O}}\subseteq F_{\mathcal{R}}$. Hence, the maximum degrees of the foreground graphs satisfy $\Delta(\mathcal{G}_{m,n}(F_{\mathcal{O}})) \le \Delta(\mathcal{G}_{m,n}(F_{\mathcal{R}}))$. Similarly, we have that $B_{\mathcal{R}} \subseteq B_{\mathcal{O}}$, and therefore the minimum degrees of the background graphs satisfy $\delta(\mathcal{G}_{m,n}(B_{\mathcal{R}})) \le \delta(\mathcal{G}_{m,n}(B_{\mathcal{O}}))$. Combining these inequalities with the identifiability inequality for video $\mathcal{R}$ yields
	\begin{align*}
	\delta(\mathcal{G}_{m,n}(B_{\mathcal{O}})) \ge{}& \delta(\mathcal{G}_{m,n}(B_{\mathcal{R}})) \\
	>{}& \frac{48}{c^2}\kappa(w^*)^4\Delta(\mathcal{G}_{m,n}(F_{\mathcal{R}})) \\
	\ge{}& \frac{48}{c^2}\kappa(w^*)^4\Delta(\mathcal{G}_{m,n}(F_{\mathcal{O}})),
	\end{align*}
	showing video $\mathcal{O}$ also satisfies the identifiability condition.
\end{proof}

It is clear that Propositions \ref{prop: embedded_connectivity} and \ref{prop: embedded_identifiability} are independent of the size, shape, and speed of a moving object. This allows us to restrict the rest of our analysis to videos with moving objects of elementary shapes, since a more complicated object may always be embedded into a larger object which covers it. In the case that the larger, simpler object is found to satisfy the conditions (\ref{eqn: connectivity}) and (\ref{eqn: identifiability}), the results of this section show the embedded object can be extracted to unique global optimality. Therefore, we will focus on rectangular moving objects for the remainder of the paper, for convenience.

\section{Conditions for Connectivity}
\label{sec: conditions_for_connectivity}
In this section, we aim to derive necessary and sufficient criteria for a video to satisfy the connectivity condition (\ref{eqn: connectivity}). We will start by defining the notion of connected backgrounds, which will assist with streamlining the proofs in Sections \ref{sec: necessary_conditions} and \ref{sec: sufficient_conditions}, in addition to granting intuitive interpretations to the conditions that follow.

\begin{definition}[Background connectivity]
	\label{def: background_connectivity}
	Given a video with frames $k\in K$ having associated background pixel sets
	\begin{equation*}
	B^{(k)} = \{(i,j)\in\Pi : S_{hk} = 0, ~ h = (j-1)d_m + i\},
	\end{equation*}
	the video is said to have a \textit{connected background} if the following two conditions are satisfied:
	\begin{enumerate}
		\item $\cup_{k\in K}B^{(k)} = \Pi$.
		\item $B_1\cap B_2 \ne \emptyset$ for all $B_1 = \cup_{k\in K_1}B^{(k)}$ and $B_2 = \cup_{k\in K_2}B^{(k)}$ such that $K_1\cup K_2 = K$.
	\end{enumerate}
\end{definition}

We now show that having a connected background is equivalent to the video's background graph $\mathcal{G}_{m,n}(B)$ being connected; videos with connected backgrounds satisfy the connectivity condition (\ref{eqn: connectivity}). This is useful, since we will use Definition \ref{def: background_connectivity} to derive simple and intuitive necessary conditions a video must satisfy in order to have a connected background (and therefore to satisfy the connectivity condition). Afterwards, we prove a sufficient condition for background connectivity, which we claim is likely satisfied for nearly any video in practice.

\begin{proposition}[Connectivity equivalence]
	\label{prop: connectivity_equivalence}
	A video's associated background graph, $\mathcal{G}_{m,n}(B)$, is connected if and only if the video has a connected background.
\end{proposition}
\begin{proof}
	The proof will proceed via contrapositive argument. We will first prove necessity.
	
	\subsubsection*{Necessity}
	Suppose that a video does not have a connected background. Then, one of the two following cases must hold:
	\begin{enumerate}
		\item $\cup_{k\in K} B^{(k)} \ne \Pi$.
		\item There exist $B_1=\cup_{k\in K_1}B^{(k)}$ and $B_2=\cup_{k\in K_2}B^{(k)}$, where $K_1\cup K_2 = K$, such that $B_1\cap B_2 = \emptyset$.
	\end{enumerate}
	Assume that the first case holds. Then, there exists a pixel $(i_0,j_0)\in\Pi$ such that $(i_0,j_0)\notin B^{(k)}$ for all $k\in K$. Therefore, we have $S_{h_0 k} \ne 0$ where $h_0=(j_0-1)d_m+i_0$, which implies
	\begin{equation*}
	(h_0,k)\notin B \textup{ for all } k\in K.
	\end{equation*}
	This shows that vertex $h_0 \in V_u$ has no incident edges in $\mathcal{G}_{m,n}(B)$, and therefore the graph is disconnected.
	
	Now, assume that the second case holds. We first note that $K_1\cap K_2 = \emptyset$, since otherwise $B_1$ and $B_2$ cannot be disjoint. Now, $B_1\cap B_2 = \emptyset$ implies that for all pixels $(i,j)\in\Pi$, either $(i,j)\in B_1$ and $(i,j)\notin B_2$, or $(i,j)\in B_2$ and $(i,j)\notin B_1$, or $(i,j)\notin B_1$ and $(i,j)\notin B_2$. In the trivial case that some pixel $(i_0,j_0)$ is neither an element of $B_1$ nor an element of $B_2$, then $\cup_{k\in K}B^{(k)}=B_1\cup B_2 \ne \Pi$, and the first case above shows that the graph $\mathcal{G}_{m,n}(B)$ is disconnected. For pixels $(i,j)\in B_1$, we have $(i,j)\notin B^{(k)}$ for all $k\in K_2$, and therefore $S_{hk} \ne 0$ where $h = (j-1)d_m + i$, which implies
	\begin{equation*}
	(h,k)\notin B \textup{ for all } k\in K_2.
	\end{equation*}
	This shows that vertex $h\in V_u$ is not adjacent to vertex $(m+k)\in V_v$ for all $k\in K_2$. Similarly, one can show that for each $(i,j)\in B_2$, the corresponding vertex $h\in V_u$ is not adjacent to vertex $(m+k)\in V_v$ for all $k\in K_1$. Since $B_1\cap B_2=\emptyset$ and $K_1\cap K_2=\emptyset$, the bipartite graph $\mathcal{G}_{m,n}(B)$ contains at least two connected components, defined by the disjoint edge sets $\mathcal{E}_1 \subseteq \{(h,m+k) : h=(j-1)d_m+i, ~ (i,j)\in B_1,~k\in K_1\}$ and $\mathcal{E}_2 \subseteq \{(h,m+k) : h=(j-1)d_m+i,~ (i,j)\in B_2,~k\in K_2\}$. Therefore, the graph is disconnected.
	
	\subsubsection*{Sufficiency}
	Suppose that a video's associated background graph, $\mathcal{G}_{m,n}(B)$, is disconnected. Then, one of the two following cases must hold:
	\begin{enumerate}
		\item There exists a vertex with no incident edges.
		\item Every vertex has at least one incident edge.
	\end{enumerate}
	Assume that the first case holds. Then, either the isolated vertex corresponds to a pixel number $h_0$ or to a frame number $k_0$. If vertex $h_0 \in V_u$ is isolated, then $(h_0,k)\notin B$ for all $k\in K$. This implies $S_{h_0k}\ne 0$ and therefore $(i_0,j_0)\notin B^{(k)}$ for all $k\in K$, where
	\begin{equation*}
	(i_0,j_0) = \left(h_0-\left(\ceiling*{\frac{h_0}{d_m}}-1\right) d_m,\ceiling*{\frac{h_0}{d_m}}\right).
	\end{equation*}
	(Here, $\ceiling{\cdot}$ represents the ceiling operator. This formula comes from the one-to-one correspondence between a pixel $(i,j)$ and its pixel number $h$ through the vectorization of a given video frame.) Thus, $(i_0,j_0)\notin \cup_{k\in K}B^{(k)}$, which implies $\cup_{k\in K}B^{(k)} \ne \Pi$. Hence, the video does not have a connected background. On the other hand, if vertex $k_0\in K$ is isolated, then $(h,k_0)\notin B$ for all $h\in V_u$. This implies $S_{hk_0}\ne 0$ and therefore $(i,j)\notin B^{(k_0)}$ for all $(i,j)\in\Pi$. Thus, $B^{(k_0)}=\emptyset$. Define $B_1 = B^{(k_0)}$ and $B_2 = \cup_{k\in K\setminus \{k_0\}}B^{(k)}$. Then $B_1\cap B_2=\emptyset$, so again the video does not have a connected background.
	
	Now, assume the second case holds. Then, the graph contains at least two nontrivial connected components. Therefore, the set $B$, which defines the edge set of the graph, can be partitioned as $B = Q_1\cup Q_2$, where $Q_1 = \{(h,k) \in\Omega : S_{hk}=0,~h\in H_1,~k\in K_1\}$ and $Q_2 = \{(h,k) \in\Omega : S_{hk}=0,~h\in H_2,~k\in K_2\}$ are nonempty, such that $H_1= V_u\setminus H_2$ and $K_1 = K\setminus K_2$. Now, define $B_1 = \cup_{k\in K_1}B^{(k)}$. This gives
	\begin{align*}
	B_1 ={}& \cup_{k\in K_1}\{(i,j) \in \Pi : S_{hk}=0, ~ h=(j-1)d_m+i\} \\
	={}& \{(i,j) \in \Pi : S_{hk}=0, ~ h=(j-1)d_m+i, ~ k\in K_1\}.
	\end{align*}
	Now, from the partitions $Q_1$ and $Q_2$ we see that a frame $k\in K_1$ has $S_{hk}=0$ only for pixel numbers $h\in H_1$. Thus, $B_1$ can be written equivalently as
	\begin{multline*}
	B_1 = \{(i,j) \in \Pi : S_{hk}=0, \\
	h=(j-1)d_m+i, ~ k\in K_1, ~h \in H_1\}.
	\end{multline*}
	Similarly, it can be shown that by defining $B_2=\cup_{k\in K_2}B^{(k)}$, we obtain
	\begin{multline*}
	B_2 = \{(i,j) \in \Pi : S_{hk}=0, \\
	h=(j-1)d_m+i, ~ k\in K_2, ~h \in H_2\}.
	\end{multline*}
	Since $H_1\cap H_2=\emptyset$ and $K_1\cap K_2=\emptyset$, we immediately see that $B_1\cap B_2=\emptyset$. Therefore, the video does not have a connected background.
\end{proof}

Proposition \ref{prop: connectivity_equivalence} shows that the connectivity of the graph $\mathcal{G}_{m,n}(B)$ is entirely dictated by whether or not a video has a connected background. Therefore, we can use the notion of background connectivity to derive intuitive and meaningful criteria a video should satisfy in order to meet the connectivity condition (\ref{eqn: connectivity}).

\subsection{Necessary Conditions for Connectivity}
\label{sec: necessary_conditions}
From Definition \ref{def: background_connectivity}, we develop three necessary conditions for background connectivity of a video, which are intuitively interpretable in terms of properties of the video (i.e., properties of pixels and frames). These necessary conditions give simple methods for showing when a video does not have a connected background, in which case no guarantees on the global optimality of the minimization (\ref{eqn: nrpca_problem}) can be made.

\begin{proposition}[Object size]
	\label{prop: object_size}
	If a video has a connected background, then there are at most $d_md_nd_f - (d_md_n+d_f-1)$ foreground pixels in the data matrix $X$.
\end{proposition}
\begin{proof}
	Since the background graph $\mathcal{G}_{m,n}(B)$ has $m+n=d_md_n+d_f$ vertices and is connected, the number of edges $\left|B\right|$ is at least $d_md_n+d_f-1$. Therefore, $\left|F\right| = mn-\left|B\right| \le d_md_nd_f - (d_md_n+d_f-1)$.
\end{proof}

For an instance in which the upper bound given by Proposition \ref{prop: object_size} is tight, yet background connectivity is still achieved, see Example \ref{ex: graph_examples}. Perhaps the most interesting implication of this result comes from the following corollary.

\begin{corollary}
	\label{cor: resolution}
	As the video resolution and number of frames increase, the maximum relative size of recognizable objects increases.
\end{corollary}
\begin{proof}
	Since there can be at most $d_md_nd_f - (d_md_n + d_f - 1)$ foreground pixels across all frames of the video, the maximum relative size of an object can be expressed as
	\begin{equation*}
	p_{\textup{max}} \coloneqq \frac{d_md_nd_f - (d_md_n + d_f - 1)}{d_md_n + d_f - 1}.
	\end{equation*}
	As the resolution of the video increases, $d_md_n\to\infty$, and therefore
	\begin{equation*}
	\lim_{d_md_n\to\infty} p_{\textup{max}} = d_f - 1.
	\end{equation*}
	Furthermore, as the length of the video increases, $d_f\to\infty$, and therefore
	\begin{equation*}
	\lim_{\substack{d_md_n\to\infty \\ d_f \to\infty}} p_{\textup{max}} = \infty.
	\end{equation*}
	Thus, we see that the maximum permissible ratio of foreground pixels to background pixels increases with the video's resolution and number of frames, as desired.
\end{proof}

Interestingly, the maximum relative object size $p_{\textup{max}}$ also shows us that with $d_f=1$ frame (i.e., a single picture), the largest recognizable object size decreases to $p_{\textup{max}} = 0$. On the other hand, with $d_md_n = 1$ (i.e., a single pixel resolution), the largest recognizable object again decreases to $p_{\textup{max}}=0$. In other words, we cannot recognize moving objects with only one frame, even with infinite resolution, and we also cannot recognize objects with only one pixel, even with infinitely many frames. Both of these observations align with the restrictions on video properties one would expect.

\begin{proposition}[Frame connectivity]
	\label{prop: frame_connectivity}
	If a video has a connected background, then each frame contains at least one background pixel.
\end{proposition}
\begin{proof}
	Suppose that there exists a frame $k_0\in K$ that contains no background pixels, i.e., $B^{(k_0)} = \emptyset$. Then, the video's background pixel sets can be partitioned as $B_1 = B^{(k_0)} = \emptyset$ and $B_2 =  \cup_{k\in K \setminus \{k_0\}}B^{(k)}$. Thus, $B_1\cap B_2 = \emptyset$, and therefore the video does not have a connected background.
\end{proof}

Proposition \ref{prop: frame_connectivity} can be interpreted as the requirement that an object can at no point cover the entirety of the frame. This matches intuition, since a moving object surely cannot be uniquely segmented from its background in these types of frames. A similar necessary condition on the obscurement of pixels, rather than frames, is given in Proposition \ref{prop: pixel_connectivity} that follows.

\begin{proposition}[Pixel connectivity]
	\label{prop: pixel_connectivity}
	If a video has a connected background, then each pixel is a background pixel in at least one frame.
\end{proposition}
\begin{proof}
	Suppose that there exists a pixel $(i_0,j_0)\in\Pi$ that is a foreground pixel for all frames $k\in K$. Then, $(i_0,j_0)\notin B^{(k)}$ for all $k\in K$. Thus, $(i_0,j_0)\notin \cup_{k\in K}B^{(k)}$, which implies $\cup_{k\in K}B^{(k)} \ne \Pi$, and therefore the video does not have a connected background.
\end{proof}

Proposition \ref{prop: pixel_connectivity} shows that if any single pixel remains as part of the foreground throughout the video's duration, we cannot guarantee benign landscape of (\ref{eqn: nrpca_problem}). This makes sense intuitively: if part of the background remains obscured throughout the video's entirety, it appears implausible to guarantee unique and globally optimal recovery of that part of the background.

\subsection{Sufficient Conditions for Connectivity}
\label{sec: sufficient_conditions}
The necessary conditions derived in Section \ref{sec: necessary_conditions} are most useful in determining when the global optimality guarantees for (\ref{eqn: nrpca_problem}) \textit{fail} to hold. In this section, we reverse the implications to derive a simple and relatively relaxed sufficient condition for ensuring the graph $\mathcal{G}_{m,n}(B)$ is connected. This leads to our first main result.

\begin{theorem}[Common background pixel]
	\label{thm: common_background_pixel}
	Suppose that each pixel of a video is a background pixel in at least one frame. If any single pixel is a background pixel in all frames of the video, then the video has a connected background.
\end{theorem}
\begin{proof}
	Since each pixel in the video is assumed to be a background pixel in at least one frame, we have that for every $(i,j)\in\Pi$, there exists a $k\in K$ such that $(i,j)\in B^{(k)}$. This implies $\cup_{k\in K}B^{(k)} = \Pi$, so the video satisfies the first condition for background connectivity.
	
	Now, suppose that there exists a pixel $(i_0,j_0)\in\Pi$ such that $(i_0,j_0)$ is a background pixel in all frames of the video. Furthermore, assume that the background pixels are partitioned as $B_1 = \cup_{k\in K_1}B^{(k)}$ and $B_2 = \cup_{k\in K_2}B^{(k)}$, where $K_1$ and $K_2$ are any two arbitrary subsets of $K$ such that $K_1\cup K_2 = K$. Since $(i_0,j_0)\in B^{(k)}$ for all $k\in K$, it must be that $(i_0,j_0)\in B_1$ and $(i_0,j_0)\in B_2$, and therefore $B_1\cap B_2 \ne \emptyset$. Since $B_1$ and $B_2$ are arbitrary partitions, the video satisfies the second condition for background connectivity. Thus, the video has a connected background.
\end{proof}

The sufficient condition given in Theorem \ref{thm: common_background_pixel} is relaxed in the sense that many videos satisfy the property of having at least one common background pixel among all frames. These common background pixels are often found in the corners of a video, away from the ``action'' of the moving objects. Therefore, with the prior knowledge that a single pixel remains unobscured by the moving objects throughout the duration of the video, the connectedness of the video's background (and therefore the connectedness of $\mathcal{G}_{m,n}(B)$) comes at only the price of ensuring that no single pixel is obscured by foreground throughout the video's entirety. This property is instantiated later in the example of Section \ref{sec: numerical_experiments}. We now focus our attention on the identifiability inequality (\ref{eqn: identifiability}).

\section{Conditions for Identifiability}
\label{sec: conditions_for_identifiability}
Recall the identifiability condition (\ref{eqn: identifiability}). The goal of this section is to determine what properties a video and its moving objects must possess in order to satisfy this condition. We make the following assumptions.

\begin{assumption}
	\label{ass: conditions_for_identifiability}
	As supported by Propositions \ref{prop: embedded_connectivity} and \ref{prop: embedded_identifiability}, we assume that the foreground $F$ is a $p_m\times p_n$ rectangle, and that at least one frame contains the entire object. Furthermore, we define $p_f\in\mathbb{Z}_{++}$ to be the maximum number of frames any pixel is obscured by the object. (Note that this can be directly computed for a variety of simple trajectories; see Remark \ref{rem: constant_trajectory}.) We assume that there exists a black pixel in at least one frame, i.e., $X_{hk}=X_{\textup{black}}$ for some $(h,k)\in\Omega$, and that $[X_{\textup{black}},X_{\textup{white}}]\subseteq\mathbb{R}_{++}$. We also assume $\|u^*\|_2=\|v^*\|_2$, as motivated in Section \ref{sec: problem_statement}. Additionally, we take $v^* = v_0^* 1_n$ for some $v_0^*\in\mathbb{R}_{++}$, which holds when the background remains constant through the video's duration, and approximately holds when the illumination variance is small enough. Finally, we set $d_f=d_md_n$, which turns out to be a key assumption for deriving bounds on $\kappa(w^*)$.
\end{assumption}

\begin{remark}(Data preprocessing)
	\label{rem: data_preprocessing}
	Various preprocessing techniques can be used to ensure the assumptions on the problem data. For instance, shifting each pixel value by $\Delta X\in\mathbb{R}_{++}$ ensures that $[X_{\textup{black}},X_{\textup{white}}] = [\Delta X,255 + \Delta X]\subseteq\mathbb{R}_{++}$. Furthermore, in a high-resolution video, we will typically find that $d_f<d_md_n$. In this case, the equality $d_f=d_md_n$ can be achieved by either repeating the video to increase the overall length $d_f$, or by compressing the video to lower the resolution $d_md_n$. The first approach is beneficial in the case that full-resolution video is needed, whereas the second approach lowers the problem dimension and speeds up computation. To appropriately rescale the resolution, one may set the new frame dimensions to $d_m' = \beta d_m$ and $d_n'=\beta d_n$, where $\beta = \sqrt{\frac{d_f}{d_md_n}}$. This is the approach we take in the experiments in Section \ref{sec: numerical_experiments}.
\end{remark}

We now prove the main result of this section.

\begin{theorem}[Rectangle identifiability]
	\label{thm: rectangle_identifiability}
	Suppose that a video satisfies Assumption \ref{ass: conditions_for_identifiability}. Then the video satisfies the identifiability condition (\ref{eqn: identifiability}) if and only if
	\begin{equation}
	\begin{aligned}
	p_f <{}& \frac{1}{c_0}d_f, \\
	p_mp_n <{}& \frac{1}{c_0}d_md_n,
	\end{aligned} \label{eqn: conditions_for_identifiability_system}
	\end{equation}
	where $c_0=49$.
\end{theorem}
\begin{proof}
	As seen in the identifiability condition (\ref{eqn: identifiability}), there are four values to analyze: the condition number $\kappa(w^*)$, the parameter $c$, the maximum degree $\Delta(\mathcal{G}_{m,n}(F))$, and the minimum degree $\delta(\mathcal{G}_{m,n}(B))$. We will first divide the proof into four separate computations, each dedicated to one of these values, then combine the final results at the end.
	
	\subsubsection{Condition number}
	\label{sec: condition_number}
	Since $\|u^*\|_2=\|v^*\|_2$ and $v^* = v_0^* 1_n$, we have
	\begin{equation*}
	d_md_nu_{\textup{min}}^{*2} \le \|u^*\|_2^2 = \sum_{k=1}^n v_k^{*2} = d_fv_0^{*2} \le d_md_nu_{\textup{max}}^{*2},
	\end{equation*}
	and so
	\begin{equation*}
	u_{\textup{min}}^* \le v_0^* \sqrt{\frac{d_f}{d_md_n}} \le u_{\textup{max}}^*.
	\end{equation*}
	Since $d_f=d_md_n$ by Assumption \ref{ass: conditions_for_identifiability}, we find $u_{\textup{min}}^* \le v_0^* \le u_{\textup{max}}^*$. This implies $w_{\textup{min}}^* = u_{\textup{min}}^*$ and $w_{\textup{max}}^* = u_{\textup{max}}^*$, and therefore $\kappa(w^*) = \frac{u_{\textup{max}}^*}{u_{\textup{min}}^*}$. Hence,
	\begin{equation*}
	1 \le \frac{v_0^*}{u_{\textup{min}}^*} \le \kappa(w^*).
	\end{equation*}
	Furthermore, since a background pixel must satisfy  $u_hv_k \in [X_{\textup{black}},X_{\textup{white}}]$ for all $(h,k)\in\Omega$, we have $u_{\textup{min}}^* v_0^* \ge X_{\textup{black}}$ and $u_{\textup{max}}^* v_0^* \le X_{\textup{white}}$, and therefore
	\begin{equation}
	\kappa(w^*) \le \frac{X_{\textup{white}}}{X_{\textup{black}}}. \label{eqn: condition_number_upper_bound}
	\end{equation}
	Taking $[X_{\textup{black}},X_{\textup{white}}] = [\Delta X,255 + \Delta X]$ with $\Delta X \in\mathbb{R}_{++}$, as in Remark \ref{rem: data_preprocessing}, we obtain $1 \le \kappa(w^*) \le 1 + \frac{255}{\Delta X}$. Therefore, taking $\Delta X$ large enough leads to
	\begin{equation}
	\kappa(w^*) \downarrow 1. \label{eqn: condition_number}
	\end{equation}
	
	\subsubsection{Parameter $c$}
	\label{sec: parameter_c}
	Notice that the identifiability condition (\ref{eqn: identifiability}) depends on a parameter $c$. This parameter is defined in \cite{fattahi_2018} to be a value in the interval $(0,1]$ such that the following holds:
	\begin{equation}
	\bar{S}_{\bar{h}\bar{k}} + w_{\bar{h}}^*w_{\bar{k}}^* > cw_{\textup{min}}^{*2}, \quad \bar{h},\bar{k}\in\{1,2,\dots,m+n\}, \label{eqn: parameter_c_inequality}
	\end{equation}
	where $\bar{X} = \bar{S} + ww^{\top}$ and
	\begin{equation*}
	\bar{S} = \begin{bmatrix}
	0_{m\times m} & S \\
	S^{\top} & 0_{n\times n}
	\end{bmatrix} \in\mathbb{R}^{(m+n)\times (m+n)}.
	\end{equation*}
	The elements of $\bar{X}$ therefore take on four forms:
	\begin{enumerate}
		\item $\bar{h},\bar{k}\le m$: We have $\bar{X}_{\bar{h}\bar{k}} = u_{\bar{h}}^*u_{\bar{k}}^* \ge u_{\textup{min}}^{*2}$.
		\item $\bar{h},\bar{k}>m$: We have $\bar{X}_{\bar{h}\bar{k}} = v_{\bar{h}-m}^* v_{\bar{k}-m}^* = v_0^{*2} \ge u_{\textup{min}}^{*2}$.
		\item $\bar{h}\le m < \bar{k}$: We have $\bar{X}_{\bar{h}\bar{k}} = (S+u^*v^{*{\top}})_{\bar{h},\bar{k}-m} = X_{\bar{h},\bar{k}-m} \ge X_{\textup{black}} = u_{\textup{min}}^*v_0^* \ge u_{\textup{min}}^{*2}$, where $u_{\textup{min}}^*v_0^*=X_{\textup{black}}$ by Assumption \ref{ass: conditions_for_identifiability}.
		\item $\bar{k}\le m < \bar{h}$: Analogous to the case above, we again find $\bar{X}_{\bar{h}\bar{k}} \ge u_{\textup{min}}^{*2}$.
	\end{enumerate}
	Since $\bar{S}_{\bar{h}\bar{k}} + w_{\bar{h}}^* w_{\bar{k}}^* = \bar{X}_{\bar{h}\bar{k}} \ge u_{\textup{min}}^{*2} = w_{\textup{min}}^{*2}$ for all $(\bar{h},\bar{k})$, we find that (\ref{eqn: parameter_c_inequality}) is satisfied for $c<1$. Therefore, we can choose
	\begin{equation}
	c \uparrow 1. \label{eqn: parameter_c}
	\end{equation}
	
	\subsubsection{Foreground graph}
	\label{sec: foreground_graph}
	Consider the graph $\mathcal{G}_{m,n}(F)$ and let us denote the degree of a vertex in $\mathcal{G}_{m,n}(F)$ as $\deg(\cdot,F)$. Note that $\deg(h,F),~h\in V_u=\{1,2,\dots,m\}$, exactly equals the number of frames in which pixel $h$ appears as foreground. Since, by Assumption \ref{ass: conditions_for_identifiability}, the maximum number of frames in which any single pixel appears as foreground is $p_f$ frames, we have
	\begin{equation*}
	\max\{\deg(h,F) : h\in V_u\} = p_f.
	\end{equation*}
	Next, we note that $\deg(m+k,F),~ (m+k)\in V_v = \{m+1,m+2,\dots,m+n\}$, exactly equals the number of foreground pixels in frame $k$. By Assumption \ref{ass: conditions_for_identifiability}, at least one frame contains the entire object, and therefore the maximum number of foreground pixels in any given frame is
	\begin{equation*}
	\max\{\deg(m+k,F) : (m+k) \in V_v\} = p_mp_n.
	\end{equation*}
	Therefore, we find that the maximum degree of the foreground graph becomes
	\begin{equation}
	\Delta(\mathcal{G}_{m,n}(F)) = \max\left\{ p_f , p_mp_n \right\}. \label{eqn: maximum_degree}
	\end{equation}
	
	\subsubsection{Background graph}
	\label{sec: background_graph}
	Consider the graph $\mathcal{G}_{m,n}(B)$ and let us denote the degree of a vertex in $\mathcal{G}_{m,n}(B)$ as $\deg(\cdot,B)$. Since $F$ and $G$ are complements with respect to $\Omega$, we have that $\mathcal{G}_{m,n}(F)$ and $\mathcal{G}_{m,n}(B)$ are bipartite complements of one another. Hence, it must be that
	\begin{align*}
	|V_u| ={}& \deg(m+k,F) + \deg(m+k,B), \\
	|V_v| ={}& \deg(h,F) + \deg(h,B),
	\end{align*}
	for all $h\in V_u$ and $(m+k)\in V_v$. This, together with the analysis of the foreground graph above, yields
	\begin{align*}
	\min\{ \deg(h,B) : h\in V_u \} ={}& d_f - p_f, \\
	\min\{ \deg(m+k,B) : (m+k)\in V_v \} ={}& d_md_n - p_mp_n.
	\end{align*}
	Therefore, we find that the minimum degree of the background graph becomes
	\begin{equation}
	\delta(\mathcal{G}_{m,n}(B)) = \min\left\{ d_f - p_f , d_md_n - p_mp_n \right\}. \label{eqn: minimum_degree}
	\end{equation}
	
	Combining the results of the four computations above by substituting (\ref{eqn: condition_number}), (\ref{eqn: parameter_c}), (\ref{eqn: maximum_degree}), and (\ref{eqn: minimum_degree}) into (\ref{eqn: identifiability}), we find that the identifiability condition is equivalent to
	\begin{equation}
	\min\{d_f-p_f , d_md_n - p_mp_n\} > 48\max\{p_f , p_mp_n\}. \label{eqn: conditions_for_identifiability_system_intermediate}
	\end{equation}
	Since $d_f=d_md_n \eqqcolon d$, we find $\min\{d_f-p_f , d_md_n - p_mp_n\} = d - \max\{p_f,p_mp_n\}$, so this gives
	\begin{equation*}
	d > 49\max\{p_f,p_mp_n\}.
	\end{equation*}
	This is equivalent to the proposed set of inequalities (\ref{eqn: conditions_for_identifiability_system}). Hence, the conditions we provide relating the video length to the size and speed of the object are seen to be necessary and sufficient, as desired.
\end{proof}

\begin{remark}[Constant trajectory]
	\label{rem: constant_trajectory}
	Take the special case of an object moving horizontally at a constant speed of $\dot{x}\in\mathbb{R}_{++}$ pixels per frame and vertically at a constant speed of $\dot{y}\in\mathbb{R}_{++}$ pixels per frame. Then, the number of frames in which any single pixel can be considered as foreground is no more than $\ceiling{\frac{p_n}{\dot{x}}}$, and is also no more than $\ceiling{\frac{p_m}{\dot{y}}}$. Assuming the object moves a sufficient distance so as to not obscure any part of the background for the entirety of the video, we have $d_f > \min\{\ceiling{\frac{p_n}{\dot{x}}},\ceiling{\frac{p_m}{\dot{y}}}\}$, so one of the two proposed bounds is active. Hence, the maximum number of frames in which a single pixel appears as foreground becomes
	\begin{equation}
	p_f = \min\left\{\ceiling*{\frac{p_n}{\dot{x}}}, \ceiling*{\frac{p_m}{\dot{y}}}\right\}, \label{eqn: constant_trajectory}
	\end{equation}
	giving bounds directly in terms of the object's size and speed.
\end{remark}

Theorem \ref{thm: rectangle_identifiability} provides us necessary and sufficient conditions to guarantee the satisfaction of the identifiability condition (\ref{eqn: identifiability}) in terms of a rectangular object's size and trajectory in relation to the resolution and length of the video. As one's intuition may predict, smaller rectangles and longer videos relax these conditions, indicating that videos with small moving objects and many frames are inherently easier to achieve globally optimal video segmentation. Together with Theorem \ref{thm: common_background_pixel}, we can provide deterministic guarantees that the optimization problem (\ref{eqn: nrpca_problem}) used to decompose a video has benign landscape, and that the resulting decomposition is unique and globally optimal. These concepts are showcased in the following video segmentation example.

\section{Numerical Experiments}
\label{sec: numerical_experiments}
In this section, we perform moving object segmentation via NRPCA on an example video in an effort to corroborate the two main results given in Theorems \ref{thm: common_background_pixel} and \ref{thm: rectangle_identifiability}. For this experiment, we recorded five minutes of surveillance video on the UC Berkeley campus, instructing volunteer human subjects to walk up and down a set of stairs, acting as the moving object to segment. The original video is $d_f = 19623$ frames long with a resolution of $d_m = 1920$ pixels by $d_n = 1080$ pixels. We preprocess the video data so that $d_md_n=d_f$, as described in Assumption \ref{ass: conditions_for_identifiability} and Remark \ref{rem: data_preprocessing}. Therefore, the NRPCA problem takes on approximately $2d_f$ variables, whereas the popular convex PCP segmentation approach would optimize over an astronomical $d_f^2$ variables after lifting the problem to higher dimensions. We also shift the pixel values in $X$ from $\{0,1,\dots,255\}$ to the interval $\mathcal{X}=[5000,5255]$ in order to guarantee $\kappa(w^*)$ is sufficiently close to unity (in this case, $\kappa(w^*)\in[1,1.05]$ by (\ref{eqn: condition_number_upper_bound})).

In order to solve the NRPCA problem, we set $\lambda=1$ and initialize a point $w_0 = (u_0, 1_n)$, with each element of $u_0\in\mathbb{R}_{++}^m$ drawn randomly from the half-normal distribution. We then iterate using stochastic gradient descent with a learning rate of $\alpha=10^{-4}$ and momentum coefficient of $\beta=0.9$, with a projection onto the nonnegative orthant at each step of the algorithm. We find that $5000$ iterations of this algorithm takes under $4$ seconds to solve the problem on a standard laptop, and that convergence within a small neighborhood is consistently achieved. We now demonstrate the use of our main results in this experiment.

Three frames from the video are shown in the columns of Fig.\ \ref{fig: sufficient_conditions}. These frames have been cropped vertically (but not horizontally) after performing the video segmentation, in order to enlarge the moving object and to save space. An \textit{a priori} estimate of the relative object size, for example in the uncropped version of frame (c) in Fig.\ \ref{fig: sufficient_conditions}, shows that a rectangle of $p_m \approx \frac{1}{7}d_m$ and $p_n\approx \frac{1}{8}d_n$ should embed the object. Furthermore, the speed at which the subjects walk up the stairs is approximated to be $\dot{x} \approx \frac{1}{500}d_n$ pixels per frame, and therefore (\ref{eqn: constant_trajectory}) can be used to approximate $p_f$ for a single trajectory up the stairs. For the full video (with multiple trajectories up and down), we estimate this value to be $p_f \approx \frac{1}{70}d_f$. Hence, Theorem \ref{thm: rectangle_identifiability} is satisfied by our approximations. Furthermore, it is clear that Theorem \ref{thm: common_background_pixel} applies, and therefore we expect both conditions (\ref{eqn: connectivity}) and (\ref{eqn: identifiability}) to be satisfied, yielding global optimality guarantees for our video segmentation. Applying the preprocessing and solution method described in Remark \ref{rem: data_preprocessing}, we solve for $w^*$. An \textit{a posteriori} computation gives the true values of $p_mp_n = \frac{1}{100}d_md_n$ and $p_f = \frac{1}{72}d_f$, satisfying (\ref{eqn: conditions_for_identifiability_system}) as expected. The original identifiability condition (\ref{eqn: identifiability}) is also found to be satisfied with $\delta(\mathcal{G}_{m,n}(B))=19352$, $\Delta(\mathcal{G}_{m,n}(F))=271$, and $\kappa(w^*)=1.05$. The resulting segmentation defined by $w^*$ is shown in Fig.\ \ref{fig: sufficient_conditions}.

\begin{figure}[ht]
	\centering
	\subfloat{%
		\centering
		\frame{\includegraphics[width=0.32\linewidth]{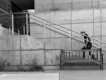}}%
	}
	\hfil
	\subfloat{%
		\centering
		\frame{\includegraphics[width=0.32\linewidth]{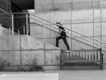}}%
	}
	\hfil
	\subfloat{%
		\centering
		\frame{\includegraphics[width=0.32\linewidth]{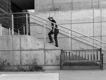}}%
	}
	\\
	\vspace*{-\baselineskip}
	\vspace*{0.02\linewidth}
	\subfloat{%
		\centering
		\frame{\includegraphics[width=0.32\linewidth]{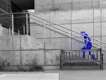}}%
	}
	\hfil
	\subfloat{%
		\centering
		\frame{\includegraphics[width=0.32\linewidth]{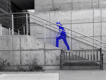}}%
	}
	\hfil
	\subfloat{%
		\centering
		\frame{\includegraphics[width=0.32\linewidth]{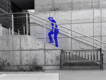}}%
	}
	\\
	\vspace*{-\baselineskip}
	\vspace*{0.02\linewidth}
	\subfloat{%
		\centering
		\frame{\includegraphics[width=0.32\linewidth]{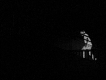}}%
	}
	\hfil
	\subfloat{%
		\centering
		\frame{\includegraphics[width=0.32\linewidth]{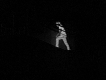}}%
	}
	\hfil
	\subfloat{%
		\centering
		\frame{\includegraphics[width=0.32\linewidth]{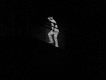}}%
	}
	\\
	\vspace*{-\baselineskip}
	\vspace*{0.02\linewidth}
	\setcounter{subfigure}{0}	
	\subfloat[]{%
		\centering
		\frame{\includegraphics[width=0.32\linewidth]{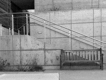}}%
	}
	\hfil
	\subfloat[]{%
		\centering
		\frame{\includegraphics[width=0.32\linewidth]{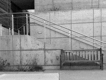}}%
	}
	\hfil
	\subfloat[]{%
		\centering
		\frame{\includegraphics[width=0.32\linewidth]{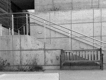}}%
	}
	\caption{Three frames, (a), (b), and (c), of the stair walking video. The first row shows the original frames with the second row showing a color overlay of their segmentations. The third and fourth rows show each frame's foreground mask and its extracted background, respectively. The performance is visually pleasing, and the algorithm accurately learned the ``difficult'' parts of the background, e.g., the bench and handrail, which create thin lines obscuring parts of the subject and can be difficult to distinguish from the moving object even by eye.}
	\label{fig: sufficient_conditions}
\end{figure}

To empirically validate the theoretical absence of spurious local solutions in this segmentation problem, we ran the optimization $N=1000$ times, obtaining the solution set $\mathcal{W}^* = \{w^{*(1)},w^{*(2)},\dots,w^{*(N)}\}$. Each run of the stochastic gradient descent algorithm used a random initial condition $w_0 = (u_0, 1_n)$ with the elements of $u_0$ again drawn from the half-normal distribution. We found the maximum relative distance between the resulting solutions to be
\begin{equation*}
\max\left\{ \frac{\|w^* - w\|_2}{\|w^*\|_2} : w\in\mathcal{W}^* \right\} = 0.0023,
\end{equation*}
indicating a 100\% success rate at converging to the same minimum $w^*$, as expected.

\section{Conclusions}
\label{sec: conclusions}
In this paper, we study the unsupervised extraction of a video's moving objects from its background via nonnegative robust principal component analysis. Although the optimization problem of interest is nonconvex, it exhibits benign landscape under certain criteria. We exploit this fact to develop conditions under which the video segmentation is unique and globally optimal. We derive these global optimality guarantees in terms of intuitive and meaningful parameters, such as the size and speed of the moving objects, as well as the length of the video. Furthermore, real video examples are given to illustrate the use of these criteria, and in what scenarios the problem's benign landscape holds.


\bibliographystyle{IEEEtran}
\bibliography{video_segmentation}

\end{document}